\documentclass{comsoc2016}

\title{Agenda Separability in Judgment Aggregation}
\author{J\'er\^ome Lang,  Marija Slavkovik,  and  Srdjan Vesic }
\usepackage{amsfonts}
\usepackage{amssymb}
\usepackage{amsmath}
\usepackage{color}
 \usepackage{graphicx}
 \usepackage[table,usenames,dvipsnames]{xcolor}

\usepackage{amsfonts}
\usepackage{amssymb}
\usepackage{amsmath}
\usepackage{color}
 \usepackage{graphicx}
 \usepackage[table,usenames,dvipsnames]{xcolor}

\newcommand{\La }{\mbox{$\mathcal{L}$} }       

\newcommand{\A}{\mathcal{A}}

\newcommand{\PA}{[\mathcal{A}]}

\newcommand{\SA}{\mathcal{A}}

\newcommand{\rest}[2]{{#1}_{\downarrow#2}}

\newcommand{\profiles}{\Pi_{\Pf_1, \Pf_2}}

\newcommand{\ai}{\varphi}

\newcommand{\Js}{J}
\newcommand{\Is}{I}

\newcommand{\Jstar}{\Js_{\ast}}

\newcommand{\Jcirc}{\Js_{\circ}}

\newcommand{\JdoubleStar}{\Js_{\ast\ast}}

\newcommand{\Istar}{\Is^{\ast}}

\newcommand{\Ct}{\Gamma}

\newcommand{\su}{\succsim_{\Pf}}

\newcommand{\Dmc}{\mathcal{J}}

\newcommand{\Pf}{P}

\newcommand{\Qf}{Q}
 
\newcommand{\Qstar}{\Qf^{\ast}}

\newcommand{\Alt}{C}

\newcommand{\m}{m}

\newcommand{\R}{R}
\newcommand{\F}{R}

\newcommand{\CMC}{\mathtt{CMC}}

\newcommand \argmax[1] {\underset{{#1}}{\mbox{argmax}}}
\newcommand \argmin[1] {\underset{{#1}}{\mbox{argmin}}}

\newcommand \ssum[1] {\underset{{#1}}{\mbox{$\sum$}}}
\newcommand \mmx[2] {\overset{#2}{\underset{{#1}}{\mbox{$\max$}}}}
\newcommand \ssm[2] {\overset{#2}{\underset{{#1}}{\mbox{$\sum$}}}}

\newcommand \ssu[1] {\succsim_{#1}}

 \newcommand{\npf}[2]{N({#1},{#2})}

\newcommand{\RMWA}{{{\textsc{med}}}}

\newcommand{\RMSA}{{{\textsc{mc}}}}
\newcommand{\RMNAC}{{\textsc{full$_H$}}}
\newcommand{\RMCSA}{{{\textsc{mcc}}}}
\newcommand{\RRA}{{{\textsc{ra}}}}

\newcommand{\RMAX}{{\textsc{R}^{d_H,\textsc{max}}}}
\newcommand{\RS}{\textsc{R}_{S}}
\newcommand{\REVS}{\textsc{R}_{rev}}

\newcommand{\mc}{{{\RMSA}}}
\newcommand{\mcc}{{{\RMCSA}}}
\newcommand{\med}{{{\RMWA}}}
\newcommand{\fullh}{\RMNAC}

\newcommand{\MCC}[1]{{ext(max(#1, |.|))}}

\newcommand{\UA}{{\mathtt{UA_{12}}}}

\newcommand{\Atoms}{{\mathtt{Atoms}}}

\newcommand{\ccb}{{\cellcolor{blue!25}}}

\newcommand{\ccy}{{\cellcolor{yellow!25}}}
\newcommand{\ccg}{{\cellcolor{gray!25}}}

\newcommand{\oas}{\textsc{oas}}

\newcommand{\as}{\textsc{as}}

\newcommand{\ra}{{\RRA}}       
\newcommand{\JA}{{\Js^1}}       
\newcommand{\JB}{{\Js^2}}

\newtheorem{definition}{Definition}
\newtheorem{proposition}{Proposition}

\newtheorem{example}{Example}
\newtheorem{corollary}{Corollary}
\newtheorem{obs}{Observation}
  \newenvironment{proof}{\noindent\emph{Proof.}}{\hfill $\Box$\break\par}

\begin{document}


\begin{abstract}
 One of the better studied properties for operators in judgment aggregation is {\em independence}, which essentially dictates that the collective judgment on one issue should not depend on the individual judgments given on some other issue(s) in the same agenda. Independence, 
 although considered 
a desirable property,
  is too strong, because together with mild additional conditions  it implies dictatorship.     We propose here a weakening of independence, named {\em agenda separability}: a judgment aggregation rule satisfies it if, whenever the agenda is composed of several independent sub-agendas,  the resulting collective judgment sets can be computed separately for each sub-agenda and then put together. We show that this property is discriminant, in the sense that among judgment aggregation rules so far studied in the literature, some satisfy it and some do not. 
We briefly discuss the implications of agenda separability on the computation of judgment aggregation rules.
\end{abstract}

\section{Introduction}

Judgment aggregation consists in finding collective judgments that are representative of a collection of individual judgments on  some logically interrelated issues.   Judgment aggregation problems originate in  political theory and public choice, however they also occur in various areas of artificial intelligence, as a  consequence of the increased distributivity of  computing systems and social networks, together with the rise of artificial agency. 
Judgment aggregation generalises voting and preference aggregation \cite{Dietrich07,ADT2013}, and has links  
with belief revision \cite{EveraereKM15,Pigozzi2006} as well as abstract argumentation \cite{CaminadaP11,Awad2015,Booth2015,BoothAR14}. For an overview of applications of judgment aggregation in artificial intelligence see for instance the work by Grossi and Pigozzi \cite{2014Grossi} or Endriss \cite{EndrissHBCOMSOC2015}.

The main focus of research in judgment aggregation is the development and analysis of judgement aggregation operators. Numerous impossibility results -- see the survey by List and Puppe \cite{ListPuppe2009} for an overview -- have dashed the hope of finding a universally applicable operator. Consequently, the suitability  of an operator for a given judgment aggregation problem has to be identified with respect to the desirable properties that the aggregation process should satisfy.

 One of the better studied properties for operators in judgment aggregation is the {\em independence} property, which essentially dictates that the collective judgment on any one issue in the agenda should not depend on the individual judgments given on any of the other issues in the same agenda. Independence is a desirable property because, among other reasons, it is a necessary condition for  strategyproofness \cite{DL05}, and it leads to rules that are both conceptually simple and easy to compute.
However,  independence is too strong; in particular, together with mild additional conditions, it implies dictatorship \cite{Dietrich07}. 

We propose  a  natural weakening of independence, named {\em agenda separability}. A judgment aggregation rule satisfies
it  if, whenever the agenda is composed of several independent sub-agendas (with an extreme form of independence being when the sub-agendas are syntactically unrelated to each other), the resulting collective judgment sets can be computed separately for each sub-agenda and then put together. Resorting to syntactically independent sub-languages is reminiscent of Parikh's language splitting \cite{Parikh99}, where decomposing a logical theory into several subtheories over disjoint sub-languages simplifies many tasks in knowledge representation, such as belief change \cite{peppas2007distance} or inconsistency handling \cite{chopra1999inconsistency}.  

The agenda separability property is very intuitive and motivations for it can be easily found. For instance, in computational linguistics, we may want to aggregate annotations from several agents about parts of texts \cite{KrugerEtAlAAMAS2014}; then, finding collective annotations about parts of two unrelated texts can (and should) be performed independently. When a rule satisfies agenda separability, it also becomes computationally simpler when applied to decomposable agendas, because the rule can be applied independently to every subagenda of the decomposition.  Agenda separability also offers a weak form of strategyproofness: no agent is able to influence the outcome on some issue from one subagenda of the partition by strategically reporting judgments about another subagenda. 
 
Of course,  a weakening of independence is meaningful only if there are  rules that satisfy it. Not only we show that this is the case, but we also show that agenda separability is discriminant, in the sense that among the known judgment aggregation rules, some satisfy it and some do not. This leads us to see agenda independence as a possible means of choosing a judgment aggregation rule against another. 

The paper is structured as follows. Section~\ref{sec:preliminaries} introduces the background. Section~\ref{sec:relwrk} discusses the independence property. In Sections~\ref{sec:agenda} and \ref{oas} we define two notions of agenda separability, and we identify some rules that satisfy them and some that do not.  Section \ref{sec:summary} contains a summary and discussion.

\section{Preliminaries}\label{sec:preliminaries}

Let $\La$ be a set of well-formed propositional logical formulas, including $\top$ (tautology) and $\bot$ (contradiction).  
An {\em issue} is a pair of formulas $ \ai, \neg \ai$ where $\ai\in\La$ and $\ai $ is neither a  tautology nor a contradiction. 
 An  {\em agenda} $\A$ is a finite set of issues and has the form $\A  = \{\ai_1,\neg \ai_1, \ldots, \ai_m, \neg \ai_m\}$. The {\em preagenda} $\PA$ associated with $\A$ is  $\PA = \{ \ai_1, \ldots, \ai_m\}$.  
  A {\em sub-agenda} is a subset of issues from $\A$.   A {\em sub-preagenda} is a subset of $\PA$. 
An agenda usually comes with an {\em integrity constraint} $\Ct$, which is a consistent formula whose role is to filter out inadmissible judgment sets. $(\A, \Ct)$ is called a {\em constrained agenda}.
As a classical example, given a set of candidates  $\Alt = \{x_1, \ldots, x_m\}$, the {\em preference agenda} over $\Alt$ \cite{Dietrich07} is $\A_C = \{x_iPx_j | 1 \leq i < j \leq m\}$, and the associated integrity constraint is $\Ct_C = \bigwedge_{i,j,k} \left( x_iPx_j \wedge x_jPx_k \rightarrow x_iPx_k \right)$. When $\Ct$ is not specified, by default it is equal to $\top$.

 A {\em  judgment} on $\ai \in \PA$ is one of $\ai$ or $\neg \ai$.  A {\em judgment set} $\Js$ is a subset of $\A$.
$\Js$ is {\em complete} iff  for each $\ai\in\PA$, either $\ai\in\Js$ or $\neg\ai\in\Js$.
 A judgment set $\Js$ (and in general, a set of propositional formulas) is {\em $\Ct$-consistent} if and only if  $\Js \cup \{\Ct\} \nvDash \bot$. 
  Let $\Dmc_{\A,\Ct}$ be the set of all {\em complete and consistent} judgment sets.
  To lighten the notations, we will generally say that a judgment set  is {\em consistent} instead of $\Ct$-consistent, and note 
$\Dmc_{\A}$ instead of $\Dmc_{\A,\Ct}$. 
 
 A {\em profile} $\Pf =\langle\Js_1, \ldots, \Js_n \rangle \in \Dmc^n_{\A}$ is a collection of  complete and  consistent individual judgment sets.  We further define $ N(\Pf,\ai) =  | \{ i \mid \varphi \in \Js_i\}|$ to be the number of all agents in $\Pf$  whose judgment set includes $\varphi$. The order $\ssu{\Pf}$ is the weak order over $\A$ defined by  $\ai \ssu{\Pf} \phi$   if and only if  $\npf{\Pf}{\ai} \geq \npf{\Pf}{\phi}$.

The {\em restriction} of  $\Pf = \langle \Js_1, \ldots, \Js_n \rangle$  over a sub-agenda  $\SA_1$ of $\SA$ is defined as 
 $\rest{\Pf}{\SA_1} = \langle \Js_1 \cap \SA_1, \ldots, \Js_n \cap \SA_1 \rangle$. 
 
Every consistent subset of the agenda  $S \subset \A$ can be extended in order to obtain a complete judgment set (there might be several such extensions). For a set $S$ of subsets of agenda , we define $ext(S) = \{ \Js \in \Dmc_{\A} \mid \mbox{ there exists } \Js' \in S \mbox{ such that } \Js' \subseteq \Js \}$.
 
A  {\em judgment aggregation rule}, for $n$ agents,  is a function $\F$ that maps any  constrained agenda $(\SA,\Ct)$   and any profile $\Pf \in \Dmc^n_{\A,\Ct}$ to a non-empty set of complete consistent judgment sets over $\SA$.\footnote{The reason why the (constrained) agenda is an argument of rules is that the notions we study need a rule to be applied to a variable agenda.   We omit writing   $\A,\Ct$      as an argument of $\F$ when defining $\F$ to improve the readability of the text.}
If  $\F$ always outputs a singleton then it is called a {\em resolute} rule.   The majoritarian judgment set associated with profile $\Pf$ contains all elements of the agenda that are supported by a majority of judgment sets in $\Pf$:
$m(\Pf)=\{\ai\in\A\mid N(\Pf,\ai) >\frac{n}{2}\}.$
A profile $\Pf$ is  {\em  majority-consistent} iff $m(\Pf)$ is consistent.

 Let $S \subseteq \La$.
We define $\Atoms(S)$ as the set of all  propositional variables appearing in $S$. For example, $\Atoms(\{ p, q \wedge r, \neg s \rightarrow \neg \neg p \}) = \{ p, q, r, s \}$. 

Given a set of formulas  $S$ and a formula $\Ct$, $S'\subseteq S$ is $\Ct$-consistent if $S' \cup \{\Ct\}$ is consistent, $S'$ is a maximal $\Ct$-consistent subset of $S$,  if $S'$ is $\Ct$-consistent and there is no $S''\supset S'$, $S'' \subseteq S$ that is $\Ct$-consistent. . We use   $max(S, \subseteq)$ to denote the maximal consistent subsets of $S$.
 The set  $S' \subseteq S$ is a maxcard 
$\Ct$-consistent subset of $S$ if $S'$ is  $\Ct$-consistent and there exists no $\Ct$-consistent set $S'' \subseteq S$ such that $|S'| < |S''|$. We use  $max(S, |.|)$ to denote the maxcard consistent subsets of $S$.

We now give the definitions of seven  judgment aggregation rules. They come from various places in the literature, where they sometimes appear with different names  \cite{TARK11,ADT2013,NehringPivato2011,NehringPP14,MillerOsherson08,EKM13}. 

Throughout the subsection,  $\Pf =\langle\Js_1, \ldots, \Js_n \rangle$ is a profile.  For two consistent and complete judgment sets $\Js, \Js'$ we denote their Hamming distance  as $d_H(\Js,\Js') = |\Js \setminus \Js'|$.

\begin{description}
\item[$\RMSA, \RMCSA$.] The maximum Condorcet rule ($\RMSA$) and the maxcard Condorcet rule ($\RMCSA$)  rules are defined as follows.
For every agenda $\A$, for every profile  $\Pf \in \Dmc^n_{\A}$,
$ \RMSA (\Pf) = \{ext(S) \mid S\in max(m(\Pf), \subseteq)\}$   and  \linebreak  $\RMCSA (\Pf)   = \{ext(S) \mid S\in max(m(\Pf), |.|)\}$. 
\item[$\RRA$.] 

For  $\A = \{\psi_1, \ldots, \psi_{2m}\}$ and a permutation $\sigma$ of $\{1,\ldots, 2m\}$, let $>_\sigma$ be the linear order on $\A$ defined by $\psi_{\sigma(1)}>_\sigma...>_\sigma \psi_{\sigma(2m)}$. We say that $>_{\sigma}$ is  compatible with $\su$ if $\psi_{\sigma(1)} \su...\su \psi_{\sigma(2m)}$. The ranked agenda rule $\RRA$ is defined as $\Js \in \RRA(\Pf)$ if and only if there exists a permutation $\sigma$ such that $>_\sigma$ is compatible with $\su$ and such that $\Js = \Js_\sigma$ is obtained by the following procedure:

 {
$\bullet$ ~$S := \emptyset$;\\
$\bullet$ ~   for $j = 1, \ldots, 2m$ do \\
$\bullet$ ~~~~~ if $S \cup \{\psi_{\sigma(j)} \}$ is consistent, let $S := S \cup \{ \psi_{\sigma(j)}  \}$; \\
$\bullet$ ~$\Js_\sigma := S$.
}

\item[$\RMAX (\Pf)$]$ = \argmin{\Js \in \Dmc_{\A}}\; \mmx{i=1}{n}\; d_H(\Js_i, \Js)$.

\item[$\RS$.] A {\em scoring function} \cite{Dietrich:2013} is defined as $s:~\Dmc_{\A}~\times~\A~\rightarrow~\mathbb{R}^+$. Given a scoring function $s$, the judgment aggregation rule $R_s$ is defined as  
 $\RS(\Pf) = \argmax{\Js \in \Dmc_{\A}}\; \ssum{\Js_i \in \Pf}\;\ssum{\varphi \in \Js}\; s(\Js_i,\varphi)$.  
If we choose the {\em reversal} scoring function $s_{rev}(\Js_i,\varphi)$ as the minimal number of judgment reversals needed in $\Js_i$ in order to reject $\varphi$ then we get the {\em reversal scoring rule} $\REVS$ \cite{Dietrich:2013}. 
 If we choose the scoring function $s$ defined by $s_{med}(\Js_i,\varphi) = 1$ if $\varphi \in J_i$ and $0$ if $\varphi \notin J_i$ then $R_s$ is exactly the {\em median} rule, i.e.\ $R_s \equiv \RMWA$.
 
 \item[$\RMWA(\Pf)$] =
$$\argmax{\Js \in \Dmc_{\A}}~\sum_{\ai \in \Js}~\npf{\Pf}{\ai}~=~\argmin{\Js \in \Dmc_{\A}}~\sum_{\Js_i\in\Pf}d_H(\Js_i, \Js).$$

\item[$\RMNAC$.] Given profiles 
$\Pf=\langle \Js_1,\ldots, \Js_n\rangle$ and $\Qf=\langle \Js^{\prime}_1,\ldots, \Js^{\prime}_n\rangle$ in $\Dmc^n_{\A}$,
let
$D_H(\Pf,\Qf) =  \ssm{i=1}{n}d_H(\Js_i,\Js^{\prime}_i)$.
$\RMNAC(\Pf) = \{ ext(m(\Qf)) \mid \Qf \in \argmin{\Qf' \in \Dmc^n_{\A}}\;D_H(\Pf,\Qf') \}$.
\end{description}

The rules defined here are irresolute, but  similarly as in voting theory, can be made resolute  by composing them with a tie-breaking mechanism.  
A simple way of defining a tie-breaking mechanism $\theta$ is via a priority relation $>_\theta$ over 
consistent and complete judgment sets. 
Given an irresolute rule $\F$ and a tie-breaking mechanism $\theta$, 
the resolute rule $\F_\theta$ is the rule that, given $\Pf$, returns the maximal (with respect to $>_\theta$) element of $\F(\Pf)$. 

\section{Relaxing Independence}\label{sec:relwrk}

A judgment aggregation rule $F$ satisfies {\em independence of irrelevant alternatives} (IIA) if 
for every two profiles $\Pf, \Pf' \in \Dmc^n_{\A}$, and every $\ai \in \A$,  if  $\rest{\Pf}{\{\ai,\neg \ai\}} = \rest{\Pf'}{\{\ai,\neg \ai\}}$, then $\varphi \in F(\Pf)$ iff  $\varphi \in F(\Pf')$.
Independence is a very strong property:  together with three seemingly innocuous properties, namely universal domain ($F$ is defined for every profile), unanimity principle,
 and collective rationality ($F$ outputs complete and consistent judgment sets),
 it implies dictatorship \cite{Dietrich07}. 
 
 In \cite{Mongin2008} a relaxation of IIA is proposed, called  {\em Independence of Irrelevant Propositional  Alternatives} (IIPA). IIPA is the requirement that for every   $\Pf, \Pf' \in \Dmc^n_{\A}$, and every $\ai \in \A$ {\em that is either an atom or a negation of an atom},  if  $\rest{\Pf}{\{\ai,\neg \ai\}} = \rest{\Pf'}{\{\ai,\neg \ai\}}$, then $\varphi \in F(\Pf)$ iff  $\varphi \in F(\Pf')$.  However \cite{Mongin2008} also shows that IIPA, modulo some conditions on the agenda, is not consistent with the unanimity preservation requirement.

Now, while it is natural to expect that the individual judgments on logically related issues will influence the choice of collective judgments for those issues, it is also natural to expect that individual judgments over logically unrelated issues will have no impact on them. To illustrate this point, we give an example from a collective decision making problem that occurs in crowdcomputing. 
 
There are a lot of tasks that are rather simple for a human to do, but fairly complicated for a computer, such as labelling images, choosing the best out of several images, identifying music segments etc. These types of tasks are called  human intelligence tasks (\textsc{hit}s). Considering the task of cataloguing pictures by location,  that is outsourced as \textsc{hit}s to an unspecified, but finite, group of people. The people undertaking these tasks should label each photo in a series and also indicate reasons for their labelling. For example: the photo is of Paris ($p$) if the Eiffel tower can be seen on it ($e$) or the Triumphal arc can be seen on it ($t$); the photo is of Rome ($r$) if the Colosseum  can be seen on it ($c$)  or the Spanish Steps can be seen on it ($s$). The commissioner of the \textsc{hit}s will  aggregate the individual labelings and assign the labels that are collectively supported. The problem of finding which labelings are collectively supported can be solved as a judgment aggregation problem; see the work by Endriss and Fern\'andez \cite{EndrissFernandezACL2013} for a similar view of crowdsourcing as a judgment aggregation problem.  
Assume, for simplicity, that we have three labellers (or agents) and two pictures. Furthermore, the commissioner is only interested in whether the first photo is of Paris and whether the second one is of Rome. The problem for the first photo is represented with the agenda $[\A_1] = \{ p,e,t, e\vee t \rightarrow p\}$, while the problem for the second photo is represented with the agenda  $[\A_2] = \{ r,c,s, c \vee s \rightarrow r\}$. Observe that $\Atoms(\SA_1) \cap \Atoms(\SA_2) = \emptyset$. The agents get the pictures at the same time. Clearly, whether the first picture is of Paris or not has nothing to do with whether the second picture is of Rome or not, consequently we would expect that the collective judgments regarding issues in $\A_1$  depend only on the judgments given for these issues, but not on the individual judgments given for issues in $\A_2$.

In the next section we relax independence 
along this principle, defining a new property called {\em agenda separability}. 

\section{Agenda Separability}\label{sec:agenda}

Following the idea that only judgments on logically related issues should influence the collective judgment on each issue, 
we define agenda separability as the property requiring that when two agendas can be split into sub-agendas that are 
independent from each other, the output judgment sets can be obtained by first applying the rule on each sub-agenda separately 
and then taking the pairwise unions of judgment sets from the two resulting sets. 

A partition $\{ \SA_1, \SA_2\}$ of $\SA$ is an {\em independent partition of $\SA$} if for every $\Js^1 \in \Dmc_{\A_1}$ and $\Js^2 \in \Dmc_{\A_2}$, $\Js^1 \cup \Js^2$ is $\Ct$-consistent.\footnote{A stronger notion of independence, which makes sense only when $\Ct = \top$, is {\em syntactical agenda independence}: a partition $\{ \SA_1, \SA_2\}$ of $\SA$ is {\em syntactically independent} if  $\Atoms(\SA_1) \cap \Atoms(\SA_2) = \emptyset$. Clearly,  syntactical agenda independence implies agenda independence, because $\Atoms(\SA_1) \cap \Atoms(\SA_2) = \emptyset$ implies that $\SA_1$ and $\SA_2$ are independent.
Note that the implication is strict: for example, let $\A = \{x, \neg x, x \leftrightarrow y, \neg (x \leftrightarrow y)\} = \A_1 \cup \A_2$, $\Ct = \top$, $\A_1 = \{x, \neg x\}$ and  $\A_2 = \{x \leftrightarrow y, \neg (x \leftrightarrow y)\}$.  $\{ \SA_1, \SA_2\}$ is an independent partition of $\SA$ although $\Atoms(\SA_1) \cap \Atoms(\SA_2)  \neq \emptyset$.}

\begin{definition}[Agenda separability]
We say that rule $\R$ satisfies {\em agenda separability} ($\as$) if for every agenda $\A$, every independent
partition $\{ \A_1, \A_2 \}$ of $\A$, and all profiles $\Pf \in \Dmc^n_{\A}$, we have  
$$\R(\Pf) =  \{ \Js^1 \cup \Js^2  ~|~ 	\Js^1 \in \R(\rest{\Pf}{\SA_1}) \mbox{ and } \Js^2 \in \R(\rest{\Pf}{\SA_2}) \}.$$
\end{definition}

If $\R$ is a resolute rule, then the last line of the definition 
simplifies into $\R(\Pf) =  \R(\rest{\Pf}{\SA_1}) \cup \R(\rest{\Pf}{\SA_2}).$

Also, by associativity of $\cup$, this notion generalises to agendas that can be partitioned into a collection $\{\SA_1, \ldots, \SA_k\}$  such that for every  $\Js_1 \in \Dmc_{\A_1}, \ldots, \Js_k \in \Dmc_{\A_k}$, $\Js_1 \cup \ldots \cup \Js_k$ is consistent.
In that case, 
\[ \R(\Pf) = \left\{  \bigcup_{i=1}^k \Js^i \mathrel{\Big|} \Js^1 \in \R(\rest{\Pf}{\SA_1}), \ldots, \Js^k \in \R(\rest{\Pf}{\SA_k}) \right\}. \]
  IIA is   
defined for resolute rules only. 
We   show that agenda separability restricted to resolute rules is a weakening of IIA.

\begin{proposition}\label{iiaimpliesas}
 Any resolute judgment aggregation rule that satisfies IIA is agenda separable.  
\end{proposition}
\begin{proof}
If a resolute rule $\R$ satisfies IIA, we can write
$\R(\Pf) =  \bigcup_{i=1}^m F_i(\rest{\Pf}{\{\varphi_i, \neg \varphi_i\}})$ where $\PA = \{\varphi_1, \ldots \varphi_m\}$ and $F_1, \ldots, F_m$ are resolute rules. 
Let  $\{\A_1, \A_2 \}$ be an independent partition of $\A$. 
Without loss of generality, assume $[\A_1] = \{ \varphi_1, \ldots, \varphi_k\}$ and $[\A_2] = \{ \varphi_{k+1}, \ldots, \varphi_m\}$. We have  $\R(\Pf) = \bigcup_{i=1}^k F_i(\rest{\Pf}{\{\varphi_i, \neg \varphi_i\}}) \cup \bigcup_{i=k+1}^m F_{i}(\rest{\Pf}{\{\varphi_{i}, \neg \varphi_{i}\}}) = 
\R(\rest{\Pf}{\SA_1}) \cup \R(\rest{\Pf}{\SA_2})$.
 \end{proof}
 
We shall see that the reverse implication does not hold.  

\begin{definition}   
The scoring function $s$ is {\em separable} if for every $\A$ and every independent partition $\{\A_1, \A_2\}$ of $\A$, 
 for $i \in\{ 1,2\}$, and every $\Js \in \Dmc_{\A}$ and $\varphi \in \A_i$, we have  $s(\Js , \varphi) = s(\Js \cap \SA_i,\varphi)$.
 \end{definition}

We omit the easy proofs of the next two results. 

 \begin{proposition}\label{scoringas}
If $s$ is a separable scoring function, then $\RS$ is agenda separable. 
 \end{proposition}
 
 \begin{corollary} 
 $\RMWA$ and $\REVS$ are agenda separable.
 \end{corollary} 


 \begin{proposition}\label{prp:AS} 
 $\RMSA$, 
 $\RMCSA$,  
 $\RRA$, and $\RMNAC$ are agenda separable.  
$\RMAX$ is not agenda separable. 
 \end{proposition}

\begin{proof} 
For  $\RMSA$ and  $\RRA$, this will be a consequence of a stronger result proven in Section \ref{oas}, therefore we give a proof only for 
 $\RMCSA$ and $\RMNAC$. Let $\{\A_1, \A_2\}$ be an independent partition of $\A$. 

\noindent \textbf{$\RMCSA$.}
Denote $B_1 = m(P_1)$, $B_2 = m(P_2)$ and $B = m(P)$. 
Let $\profiles = \{ \Js^1 \cup \Js^2 ~|~ \Js^1 \in \RMCSA(\Pf_1) \mbox{ and } \Js^2 \in \RMCSA(\Pf_2)  \}.$
We first show that $\RMCSA(\Pf) \subseteq \profiles$.
Let $\Jstar \in \RMCSA(\Pf)$; thus $\Jstar \in \MCC{B}$.
Let $\Jstar^1 = \Jstar \cap \SA_1$ and $\Jstar^2 = \Jstar \cap \SA_2$. 
$\Jstar^1$ and $\Jstar^2$ are consistent, because $\Jstar$ is consistent. Assume
$\Jstar^1 \notin \MCC{B_1}$.
Since $\Jstar^1$ is consistent,
there exists
$\JdoubleStar^1 \in \MCC{B_1}$ such that $|\JdoubleStar^1| > |\Jstar^1| $. 
Let $\JdoubleStar = \JdoubleStar^1 \cup \Jstar^2$. 
Because $\{\A_1, \A_2\}$ is an independent partition of $\A$,
the consistency of $\JdoubleStar^1$ and of $\Jstar^2$ implies the consistency of $\JdoubleStar$. But then
$|\JdoubleStar|  > |\Jstar|$, which contradicts $\Jstar \in \MCC{B}$.
Therefore, $\Jstar^1  \in \MCC{B_1}$. Similarly, $\Jstar^2 \in \MCC{B_2}$. Thus, $\Jstar \in \profiles$.

Now we show that $\profiles \subseteq \RMCSA(\Pf)$.
Let $\Js^1 \in \RMCSA(\Pf_1)$ and $\Js^2 \in \RMCSA(\Pf_2)$, that is,
$\Js^1 \in \MCC{B_1}$ and  $\Js^2 \in \MCC{B_2}$. 
Let us show that 
$\Js = \Js^1 \cup \Js^2 \in \MCC{B}$. 
Because $\{\A_1, \A_2\}$ is an independent partition of $\A$,
$\Js$ is consistent.
Suppose that there exists 
$\Jstar \in \MCC{B}$ such that $|\Jstar| > |\Js|$. 
Let $\Jstar^1 =  \Jstar \cap B_1$ and $\Jstar^2 = \Jstar \cap B_2$. 
Both $\Jstar^1$ and $\Jstar^2$ are consistent, and $|\Jstar| > |\Js|$ implies that 
$|\JdoubleStar^1| > |\Jstar^1|$ or $|\JdoubleStar^2| > |\Jstar^2|$, which contradicts
$\Jstar^1 \notin \MCC{B_1}$ and $\Jstar^2 \notin \MCC{B_2}$.
Thus, it must be that $\Js \in \MCC{B}$ and, consequently, $\Js \in \RMCSA(\Pf)$. 

 \noindent \textbf{$\RMNAC$.}
Let $X \subseteq \Dmc^n_{\A}$ be the set of all profiles $\Qf$ such that  $ext(m(\Qf)) \subseteq \Dmc_{\A}$, 
$\CMC(\Pf) = \mbox{argmin}_{\Qf \in X}D_H(\Pf,\Qf)$, and
$\UA = \{ \Js^1 \cup \Js^2 ~|~ \Js^1 \in$ $\RMNAC(\rest{\Pf}{\SA_1})$  and $\RMNAC(\rest{\Pf}{\SA_2}) \} $.

We first show that $\RMNAC(\Pf) \subseteq \UA$. 
Let $\Jcirc \in \RMNAC(\Pf)$. Let $\Jcirc^1 = \Jcirc \cap \SA_1$ and 
$\Jcirc^2 = \Jcirc \cap \SA_2$. 
 Since $\Jcirc \in \RMNAC(\Pf)$ then there exists $\Qf \in \CMC(\Pf)$ such that  $\Jcirc \in ext(\m(\Qf))$. 
 Let us show that $\rest{\Qf}{\SA_1} \in \CMC(\rest{\Pf}{\SA_1})$.   
Suppose that $\rest{\Qf}{\SA_1} \notin \CMC(\rest{\Pf}{\SA_1})$. Then, there exists a majority-consistent $\Qstar_1 \in \Dmc_{\SA_1}^n$, 
$\Qstar_1 = \langle \Istar_1, \ldots, \Istar_n \rangle$  such that 
$D_H(\Qstar_1, \rest{\Pf}{\SA_1}) < D_H(\rest{\Qf}{\SA_1}, \rest{\Pf}{\SA_1})$. 
Let $\Qf = \langle \Is_1, \ldots, \Is_n \rangle$. Let $\rest{\Qf}{\SA_1} = \langle \Is_1^{1}, \ldots, \Is_n^{1} \rangle$ and
$\rest{\Qf}{\SA_2} = \langle \Is_1^{2}, \ldots, \Is_n^{2} \rangle$.  Define 
$\Qstar = \langle \Istar_1 \cup \Is_1^2, \ldots, \Istar_n \cup \Is_n^2 \rangle$. 
Because $\{\A_1, \A_2\}$ is an independent partition of $\A$,
$\Qstar$ is a majority-consistent profile. Note also that $D_H(\Qstar, \Pf) < D_H(\Qf, \Pf)$. Contradiction.
Thus,  $\rest{\Qf}{\SA_1} \in \CMC(\rest{\Pf}{\SA_1})$, and for the same reasons, 
$\rest{\Qf}{\SA_2} \in \CMC(\rest{\Pf}{\SA_2})$. Therefore, $\Jcirc^1 \in \RMNAC(\rest{\Pf}{\SA_1})$, $\Jcirc^2 \in \RMNAC(\rest{\Pf}{\SA_2})$, and
$\RMNAC \subseteq \UA$.

We now show that $\UA \subseteq \RMNAC$. Let $\Jcirc^1 \in \RMNAC(\rest{\Pf}{\SA_1})$ and $\Jcirc^2 \in \RMNAC(\rest{\Pf}{\SA_2})$.
Thus, there exist  
profiles $\Qf_1 \in \CMC(\rest{\Pf}{\SA_1})$ and $\Qf_2 \in \CMC(\rest{\Pf}{\SA_2})$
such that $\Jcirc^1 \in ext(\m(\Qf_1))$  and $\Jcirc^2 \in ext(\m(\Qf_2))$. Let $\Qf_1 = \langle \Qf_1^1, \ldots, \Qf_n^1 \rangle$,
$\Qf_2 = \langle \Qf_1^2, \ldots, \Qf_n^2 \rangle$, and $\Qf = \langle \Qf_1^1 \cup \Qf_1^2, \ldots, \Qf_n^1 \cup \Qf_n^2 \rangle.$
Because $\{\A_1, \A_2\}$ is an independent partition of $\A$,
 $\Qf$ is majority-consistent.
 
Let us show that $\Qf \in \CMC(\Pf)$.  
Assume that $\Qf \notin \CMC(\Pf)$.
Then there exists $\Qstar \in \CMC(\Pf)$ s.t. $D_H(\Qstar, \Pf) < D_H(\Qf, \Pf)$.
Observe that
$ D_H( \rest{\Qstar}{\SA_1}, \Pf ) + D_H( \rest{\Qstar}{\SA_2}, \Pf ) <  D_H( \rest{\Qf}{\SA_1}, \Pf ) + D_H( \rest{\Qf}{\SA_2}, \Pf ). $
This means that
$D_H( \rest{\Qstar}{\SA_1}, \Pf ) <  D_H( \rest{\Qf}{\SA_1}, \Pf )$ or   $D_H( \rest{\Qstar}{\SA_2}, \Pf ) < D_H( \rest{\Qf}{\SA_2}, \Pf ) $.
Recall that $\rest{\Qf}{\SA_1} = \Qf_1$ and $\rest{\Qf}{\SA_2} = \Qf_2$.
Thus, $D_H( \rest{\Qstar}{\SA_1}, \Pf ) <  D_H( \Qf_1, \Pf )$ or $D_H( \rest{\Qstar}{\SA_2}, \Pf ) < D_H( \Qf_2, \Pf )$, which, together with the fact that $\rest{\Qstar}{\SA_1}$ and  $\rest{\Qstar}{\SA_2}$ are majority-consistent, contradicts
$\Qf_1 \in \CMC(\rest{\Pf}{\SA_1})$ and  $\Qf_2 \in \CMC(\rest{\Pf}{\SA_2})$.  
Thus,   
$\Qf \in \CMC(\Pf)$. Note that $\Jcirc^1 \cup \Jcirc^2 \in \m(\Qf)$.
This implies that $\Jcirc^1 \cup \Jcirc^2 \in \RMNAC(\Pf)$.

\textbf{$\RMAX$.}
We provide a counter example.
Let $\A~=~\SA_1~\cup~\SA_2$ with 
$[\A_1] = \{ p, q,  p\wedge q \}$ and $[\A_2] = \{ t\}$.
Consider the profile $\Pf$ from Figure \ref{fig:rmax-not-as} with $\Pf_1 = \rest{\Pf}{\SA_1}$ and $\Pf_2 = \rest{\Pf}{\SA_2}$. 
 \begin{figure}[!ht]
	\centering
		\begin{tabular}{r|cccc}
				Agents & $p$ & $q$ & $p \wedge q$ & $t$  \\ \hline
				$\Js_1$ & \cellcolor{blue!25} + & \cellcolor{blue!25} + &\cellcolor{blue!25}  +   &\cellcolor{gray!25} + \\
				$\Js_2$ & \cellcolor{blue!25} + & \cellcolor{blue!25} - &\cellcolor{blue!25} -   &\cellcolor{gray!25} + \\
				$\Js_3$ & \cellcolor{blue!25} - & \cellcolor{blue!25} + &\cellcolor{blue!25} -   &\cellcolor{gray!25} -\\
				\multicolumn{1}{c}{}&\multicolumn{3}{c}{ $\Pf_1$}& $\Pf_2$\\
		\end{tabular}
 \caption{\small Counter example to $\RMAX$ being agenda separable.}
 \label{fig:rmax-not-as}
 \end{figure}
We obtain  $\RMAX(\Pf) = \{ \{ \neg p, q,\neg(p\wedge q), t\} \} $. However   $\RMAX(\Pf_2) = \{ \{t\},$ $\{\neg t\} \}$ and
$\RMAX(\Pf_1) = \{ \{\neg p, q,\neg(p\wedge q)\},  \{  p, q, (p\wedge q)\}, \{ p,\neg q,\neg(p~\wedge~q)\} \}$.
\end{proof}
 
The fact that a rule satisfies agenda separability does not imply that a resolute rule obtained by composing it with a tie-breaking mechanism satisfies agenda separability as well. 
For instance, if tie-breaking favours $\{\neg a\}$ over $\{a\}$ when $\SA = \{a\}$,  $\{\neg b\}$ over $\{b\}$ when $\SA = \{b\}$,  and $\{a,b\}$ over all other judgment sets when $\SA = \{a,b\}$,  and if $P$ contains one judgment set $\{a,b\}$ and one judgment set $\{\neg a, \neg b\}$, then for any one of our rules, and with $\A_1 = \{a, \neg a\}$ and $\A_2 = \{b, \neg b\}$, we have $\F(P_{\downarrow \A_1}) = \{\neg a\}$,  $\F(P_{\downarrow \A_2}) = \{\neg b\}$,  and $\F(P) = \{a, b\}$. However, if the tie-breaking priority relation $>_\theta$ satisfies the following decomposability property, then agenda separability of an irresolute rule implies agenda separability of its composition with $\theta$.

A tie-breaking priority relation $>_\theta$ is {\em agenda separable} if for every agenda $\A$, for every independent partition $\{\A_1, \A_2\}$ of $\A$, and every $\Jstar^1, \Jcirc^1 \in \Dmc_{\A_1}$, and $\Jstar^2, \Jcirc^2 \in \Dmc_{\A_2}$, $\Jstar^1 >_\theta \Jcirc^1$ and $\Jstar^2 >_\theta \Jcirc^2$ imply $\Jstar^1 \cup \Jstar^2  >_\theta \Jcirc^1 \cup \Jcirc^2$. 

\begin{obs}
If $>_\theta$ is an agenda separable tie-breaking priority relation and $\F$ is agenda separable, then $\F_\theta$ is agenda separable.
\end{obs}

Let $>_\theta$ be an agenda separable tie-breaking priority relation, then $\RRA_\theta$ is agenda separable. However,  since it satisfies universal domain, unanimity principle \cite{TARK11}, and collective rationality, then it does not satisfy IIA.
Hence, the implication stated in Proposition~\ref{iiaimpliesas} is strict.

Lastly, we would like to state two observations about the properties of rules that are agenda separable.
\begin{obs} Let $K$ be a constant and say that agenda $\A$ is $K$-decomposable if $\A$ can be partitioned into $p$ syntactically unrelated agendas $\A_1, \ldots, \A_p$ such that for all $i$  $|\Atoms(\A_i)| \leq K$.  If a rule satisfies agenda separability, then the collective judgment sets  can be computed in time $O(2^Kn)$ whenever the agenda is $K$-decomposable.  
\end{obs}

In other terms, computing these rules is parameterized tractable when he parameter is the degree $K$ of decomposability, which is a complexity gap, since winner determination for these rules is $\Theta^2_p$-hard or even $\Pi^2_p$-hard \cite{LangSlavkovik14,EndrissH15}. 

Moreover, agenda separability allows for a weak form of strategyproofness. 
Indeed, if A can be partitioned into $p$ syntactically unrelated agendas $\A_1, \ldots, \A_p$, then no agent is able to influence the outcome on some issue in $\A_i$ by reporting strategic judgments about issues of $\A_j$ for $j \neq i$.

\section{Overlapping Agenda Separability}\label{oas}

In this section, we consider a stricter property than agenda separability. We first need the notion of independent overlapping decomposition.

\begin{definition}[Independent overlapping decomposition]
Let $\A$ be an agenda  
and let $\A = \A_1 \cup \A_2$ (but not necessarily $\A_1 \cap \A_2 = \emptyset$). We say that $\{\A_1, \A_2\}$ is an \emph{independent overlapping decomposition} (IOD) of $\A$ if and only if
for every $\Js^1 \in \Dmc_{\A_1}$, for every $\Js^2 \in \Dmc_{\A_2}$
\[\mbox{ if ~~} \Js^1 \cap \A_2 = \Js^2 \cap \A_1 \mbox{~~ then ~~} \Js^1 \cup \Js^2 \in \Dmc_{\A}.\]  
\end{definition}

\begin{example}
\label{ex:ios}
Let $\PA = \{p, \neg p \vee t, p \leftrightarrow q\}$, 
$[\SA_1] = \{p, \neg p \vee t\}$ and $[\SA_2] = \{\neg p \vee t, p \leftrightarrow q\}$. Note that $\{\A_1, \A_2\}$ is an independent overlapping decomposition of $\A$.
\end{example}

\begin{obs}
\label{obs:3}
Every independent partition is an independent overlapping decomposition. 
\end{obs}

Example \ref{ex:ios} shows that the contrary of the previous observation does not hold. Indeed, as soon as the intersection of the two sub-agendas is non-empty, they do not form an independent partition. 

There is a clear connection between independent overlapping decompositions and conditional independence in propositional logic \cite{Darwiche1997,Lang2002}; we do not give technical details here, but we mention that this connection gives us several characterizations as well as complexity results for finding independent overlapping decompositions. 

We can now introduce the definition of overlapping agenda separability.

\begin{definition}[Overlapping agenda separability]
We say that rule $\R$ satisfies \emph{overlapping agenda separability} (OAS) if for every agenda $\A$  
and every independent overlapping decomposition $\{\A_1, \A_2\}$ of $\A$, for every profile $\Pf$ over $\A$ it holds that:
$\mbox{ if for every } \Js^1 \in \R(\rest{\Pf}{\A_1}), \mbox{ for every } \Js^2 \in \R(\rest{\Pf}{\A_2}), \mbox{ we have } \Js^1 \cap \A_2 = \Js^2 \cap \A_1$ $ \mbox{ then } \R(\Pf) =  \{ \Js^1 \cup \Js^2  ~|~ 	\Js^1 \in \R(\rest{\Pf}{\SA_1}) \mbox{ and } \Js^2 \in \R(\rest{\Pf}{\SA_2})\}.$
\end{definition}

\begin{obs}\label{obs:asoas}
Overlapping agenda separability implies agenda separability. 
\end{obs}

\begin{proof}
Let $\{\A_1, \A_2\}$ be an independent partition of $\A$. From Observation \ref{obs:3} $\{\A_1, \A_2\}$ is an IOD. Since $\A_1 \cap \A_2 = \emptyset$, condition $\Js^1 \cap \A_2 = \Js^2 \cap \A_1$ is satisfied for every $\Js^1$, $\Js^2$. Thus, $\R(\Pf) =  \{ \Js^1 \cup \Js^2  ~|~ 	\Js^1 \in \R(\rest{\Pf}{\SA_1}) \mbox{ and } \Js^2 \in \R(\rest{\Pf}{\SA_2})\}$.
\end{proof}
\begin{proposition}\label{prp:MSA-RA} $\RMSA$ and $\RRA$ satisfy \oas.
\end{proposition}
\begin{proof}\linebreak
\noindent
\textbf{$\RMSA$.}
Suppose that for every  $\Js^1 \in \mc(\Pf_1), \mbox{ for every } \Js^2 \in \mc(\Pf_2), ~~ \Js^1 \cap \A_2 = \Js^2 \cap \A_1$.
Let $\Pi_{\Pf_1,\Pf_2} = \{\Js^1 \cup \Js^2 \mid \Js^1 \in \mc(\Pf_1) \mbox{ and }  \Js^2 \in \mc(\Pf_2)   \}$.

We first show that $\mc(\Pf) \subseteq \Pi_{\Pf_1,\Pf_2}$. Let $\Js \in \mc(\Pf)$. Denote $\Js^1 = \Js \cap \A_1$ and $\Js^2 = \Js \cap \A_2$. We claim that $\Js^1 \in \mc(\Pf_1)$ and   $\Js^2 \in \mc(\Pf_2)$. Note that $\Js^1$ and $\Js^2$ are consistent. By means of contradiction, and without loss of generality, assume $\Js^1 \notin \mc(\Pf_1)$. Thus, there exists $\Js_{\star}^1 \in \Dmc_{\A_1}$ such that
$ \Js^1 \cap \m(\Pf)  \subset \Js_{\star}^1 \cap \m(\Pf).  $
Denote $\Js_\star = \Js_{\star}^1 \cup \Js^2$. Observe that $\Js_\star$ is consistent. Furthermore,  $\Js \cap \m(\Pf) \subset \Js_\star \cap \m(\Pf)$, thus $\Js \notin \mc(\Pf)$, contradiction. 

We now show that $\Pi_{\Pf_1,\Pf_2} \subseteq \mc(\Pf)$. Let $\Js^1 \in \mc(\Pf_1)$ and $\Js^2 \in \mc(\Pf_2)$. Denote $\Js = \Js^1 \cup \Js^2$. Since $\{\A_1, \A_2\}$ is an IOD, $\Js$ is consistent.  
Suppose $\Js \notin \mc(\Pf)$. Thus, there exists $\Js' \in \Dmc_{\A}$ such that     
$\Js \cap \m(\Pf) \subset \Js' \cap \m(\Pf). $
Let $\varphi \in (\Js' \cap \m(\Pf)) \setminus (\Js \cap \m(\Pf) )$. Without loss of generality, assume $\varphi \in \A_1$. 
Denote $\Js_\star^1 = \Js' \cap \A_1$. Note that $\Js_\star^1$ is consistent and $\Js^1 \cap \m(\Pf) \subset \Js_\star^1 \cap \m(\Pf)$, contradiction.  

\textbf{$\RRA$.} We give only a proof sketch.  

Suppose that  for every $ \Js^1 \in \ra(\Pf_1), \mbox{ for every } \Js^2 \in \ra(\Pf_2), ~~ \Js^1 \cap \A_2 = \Js^2 \cap \A_1. $
Let $\Pi_{\Pf_1,\Pf_2} = \{\Js^1 \cup \Js^2 \mid \Js^1 \in \ra(\Pf_1) \mbox{ and }  \Js^2 \in \ra(\Pf_2)   \}$.
Let $\JA \in \ra(\Pf_1)$, $\JB \in \ra(\Pf_2)$. Denote $\Js = \JA \cup \JB$. We claim that $\Js \in \RRA(\Pf)$.
Because  $\JA \in \ra(\Pf_1)$, there is an order $>_{\sigma^1}$ on $\SA_1$, refining $\ssu{\Pf_1}$ such that $\JA = \Js_{\sigma^1}$. Similarly, there is an order $>_{\sigma^2}$ on $\A_2$, refining $\ssu{\Pf_2}$, such that $\JB = \Js_{\sigma^2}$.
We first claim that without loss of generality, we can assume that $>_{\sigma^1}$ and  $>_{\sigma^2}$ coincide on $\A_1 \cap \A_2$. For this we construct $\sigma''$ on $\A_2$, refining $\ssu{\Pf_2}$, such that $>_{\sigma''}$ coincides with $>_{\sigma^1}$ on $\A_1 \cap \A_2$ and $\Js_{\sigma''} = \Js_{\sigma^2} = \JB$.

Now, let $>_{\sigma}$ be an order on $\A$ refining $\ssu{\Pf}$ and extending both $>_{\sigma^1}$ and $>_{\sigma^2}$. Let $\A = \{\alpha_1, \ldots, \alpha_{2m}\}$. Without loss of generality, suppose $\alpha_1 >_{\sigma} \ldots >_{\sigma} \alpha_{2m}$. Let $S_i \subseteq \A$ be the set obtained at the step $i$ of construction of $\Js = \Js_{\sigma}$. We show by induction on $i$ that 
\[(H_i) ~~~~ \forall j \in \{1, \ldots, i\} ~~\mbox{ we have } ~~ \alpha_{j} \in S_i \mbox{ iff } \alpha_j \in \Js^1 \cup \Js^2.  \]
From $(H_{2m})$, we obtain  $\Js = \Js^1 \cup \Js^2$. 

We now show that $\ra(\Pf) \subseteq \Pi_{\Pf_1,\Pf_2}$. Suppose $\Js \in \ra(\Pf)$. Let $>_{\sigma}$ be an order on $\A$ such that $\Js = \Js_{\sigma}$. Without loss of generality, suppose $\alpha_1 >_{\sigma} \ldots >_{\sigma} \alpha_{2m}$. 
Denote by $>_{\sigma^1}$ (resp. $>_{\sigma^2}$) the restriction of $>_{\sigma}$ on $\A_1$ (resp. $\A_2$). Let $\Js^1 = \Js_{\sigma^1}$ and $\Js^2 = \Js_{\sigma^2}$. 
Observe  that $J^1 \cap \A_2 = J^2 \cap \A_1$. Since $\{\A_1, \A_2\}$ is an IOD, $J^1 \cup J^2$ is consistent. 

Let $S_i \subseteq \A$ be the set obtained at the step $i$ of construction of $\Js = \Js_{\sigma}$. We show by induction on $i$ that 
\[(H_i) ~~~~ \forall j \in \{1, \ldots, i\} ~~\mbox{ we have } ~~ \alpha_{j} \in S_i \mbox{ iff } \alpha_j \in \Js^1 \cup \Js^2.  \]
By putting $i = 2m$, we obtain $J = J^1 \cup J^2$.   
\end{proof}

\begin{proposition}\label{prp:oasALL} $\RMCSA$, $\RMWA$, $\RMNAC$, and  
$\REVS$ do not satisfy \oas.
\end{proposition}

\begin{proof}\linebreak 
\noindent \textbf{$\mcc$}, \textbf{$\med$} and \textbf{$\fullh$}.
We now provide a counter-example to show that $\mcc$ and $\med$ do not satisfy \oas. Let $[\A_1] = \{ p, p \rightarrow q, p \rightarrow r, q, r \}$, $[\A_2] = \{q, r, s, s \rightarrow q, s \rightarrow r \}$, and $\A = \A_1 \cup \A_2$. Observe that $\{\A_1, \A_2\}$ is an IOD of $\A$. 
\begin{figure}[!ht]
\centering
\begin{tabular}{r|cccccccc}
\tiny
{\tiny }&{\footnotesize   $p$}  & {\footnotesize $p \rightarrow q$}    &{\footnotesize $p \rightarrow r$ }&{\footnotesize  $q$} & {\footnotesize  $r$}  &{\footnotesize   $s$}  & {\footnotesize $s \rightarrow q$ } & {\footnotesize $s \rightarrow r$}  \\ \hline
{\footnotesize$\Js_1$} &  \ccb + &  \ccb + &  \ccb + &  \ccg + &  \ccg + &  \ccy + &  \ccy + &  \ccy +  \\  
{\footnotesize$\Js_2$} &  \ccb - &  \ccb + &  \ccb + &  \ccg - &  \ccg - &  \ccy - &  \ccy + &  \ccy +  \\  
{\footnotesize$\Js_3$ }&  \ccb + &  \ccb - &  \ccb - &  \ccg - &  \ccg - &  \ccy + &  \ccy - &  \ccy -  \\  
		\end{tabular}
\caption{The counter example used to show that several rules do not satisfy \oas.}
	\label{counter-ex-mcc-oas}
\end{figure}
Consider the profile from \mbox{Figure \ref{counter-ex-mcc-oas}.} 
We obtain 
$\mcc(\Pf_1) = \med(P_1) = \fullh(P_1) = \left\{
\begin{array}{lcccccl}
\{&\neg p,& p\rightarrow q,&  p\rightarrow r, & \neg q, & \neg r, \}\\
\end{array}
\right\}, 
$ 
and
$
\mcc(\Pf_2) = \med(P_2) = \fullh(P_2) = \left\{
\begin{array}{lcccccl}
\{&\neg s,& s\rightarrow q,& s \rightarrow r, & \neg q, & \neg r, \}\\
\end{array}
\right\}.
$
However, $\mcc(\Pf) = \med(P) = \fullh(P) = 
 \{
\{ \neg p,    p\rightarrow q,    p\rightarrow r,   \neg q,   \neg r,   \neg s,   s\rightarrow q,   s \rightarrow r\},$ 
$\{ p,   p\rightarrow q,     p\rightarrow r,  q,   r,     s,   s\rightarrow q,   s \rightarrow r\}
 \}.$
 
\textbf{$\REVS$.} The proof is omitted due to space limitations.
\end{proof}
The preference agenda  \cite{DietrichList07} associated with a set of alternatives $\Alt=\{x_1,\ldots, x_q\}$ is  $\A_{\Alt}=\{ x_iPx_j\mid1\leq i < j\leq q\}$. When $j>i$, $x_iPx_j$ is not a proposition of $\A_{\Alt}$, but we write $x_j P x_i$ as a shorthand for $\neg (x_jPx_i)$. Conversely, given a  judgment set $\Js$ on $\A_{\Alt}$, the binary relation $\succ_{\Js}$ over $\Alt$ is defined by: for all $x_i, x_j \in \Alt$,  $x_i \succ_{\Js} x_j$ if $x_iPx_j \in \Js$ and $x_j \succ_{\Js} x_i$ if $\neg x_iPx_j \in \Js$.
\begin{obs}\label{obspa}
For any $m \geq 3$, there exists no (non-trivial) independent overlapping decomposition of the preference agenda over  the set of alternatives   $\Alt=\{x_1, \ldots, x_m\}$.
\end{obs}

\begin{proof}
We first establish the following lemma: if $\{\A_1, \A_2\}$ is an independent overlapping decomposition, then for all $x_i, x_j, x_k$, $x_iPx_j$ and $x_iPx_k$ are both in $\A_1$ or both in $\A_2$. Assume that it is not the case:  without loss of generality, $x_i P x_j \in \A_1$ and $x_i P x_k \in \A_2$. Also without loss of generality, assume $x_jPx_k \in \A_1$. Let $\Js_1$ and $\Js_2$ be two consistent judgment sets over $\A_1$ and $\A_2$ such that $\Js_1$ contains $\{x_iPx_j, x_jPx_k\}$, $\Js_2$ contains $x_kPx_i$, and $\Js_1$ and $\Js_2$ are completed in an arbitrary way such that $\Js_1 \cap \A_2 = \Js_2 \cap \A_1$; $\Js_1 \cup \Js_2$ is an inconsistent judgment set over $\A_1 \cup \A_2$, which contradicts the assumption that $\{\A_1, \A_2\}$ is an independent overlapping decomposition.

Assume without loss of generality that $x_1Px_2 \in \A_1$. Let $x', x'' \in \{x_1, \ldots, x_k\}$. If $x' = x_1$ or $x'' = x_1$ then the above lemma implies that $x'Px'' \in \A_1$. If neither $x' = x_1$ nor $x'' = x_1$, then the above lemma implies that $x_1Px' \in \A_1$, and applying the lemma again leads to $x'Px'' \in \A_1$. This being true for all $x', x''$, we have $\A_1 = \A$, and  $\{\A_1, \A_2\}$ is a trivial  decomposition. 
\end{proof}

\section{Discussion}\label{sec:summary}
We proposed a new property for judgment aggregation, namely agenda separability. It is a relaxation of the classical independence property, and unlike it, it is satisfied by several non-degenerate judgment aggregation rules. We have defined a stronger version of agenda separability, namely overlapping agenda separability, which is even more discriminant, since we have identified only two of the previously studied judgment aggregation rules that satisfy it, namely $\RMSA$ and $\RRA$. Note that $\RRA$ satisfied furthermore unanimity principle \cite{TARK11}. Also, two rules were left out of this paper due to space limitations: the judgment aggregation version of the Young rule, which does not satisfy agenda separability, and the `geodesic' distance-base rule of Duddy and Piggins \cite{DuddyP:2012}, which satisfies agenda separability but not overlapping agenda separability.

%

A possible reason why agenda separability has not been studied sooner is that it is not applicable to common agendas such as the preference agenda, simply because they are not decomposable  (cf. Observation \ref{obspa}).  A similar observation would hold for other agendas of interest, such as those used for the aggregation of equivalence relations or for committee elections. However, agenda separability does apply to variants of these problems. Suppose for instance that we have to elect a committee made of  $K$ men and $K$ women; then agenda separability applies and says that the election of the $K$ men and the $K$ women do not interfere.

This notion of agenda separability should not be confused with a notion of separability, also known as consistency or reinforcement, considered in voting theory \cite{Young1975} and generalized to judgment aggregation \cite{TARK11}: these notions say that if a {\em profile} $P$ can be decomposed into two subprofiles $\Pf_1$ and $\Pf_2$   for which the output is the same, then this should also be the output for $\Pf$. 

An ambitious issue for further work would be characterizing the set of rules that satisfy agenda separability, or one of its variants. 

\paragraph{Acknowledgements.} We would like to thank the anonymous reviewers for helping us to improve this work. J\'er\^ome Lang is supported by the ANR project CoCoRICo-CoDec.

\bibliographystyle{plain}
\bibliography{aaai}

\end{document}